%% file: main.tex
\documentclass[conference]{IEEEtran}
\IEEEoverridecommandlockouts
\usepackage{cite}
\usepackage{amsmath,amssymb,amsfonts}
\usepackage{graphicx}
\usepackage{textcomp}
\usepackage{xcolor}
\def\BibTeX{{\rm B\kern-.05em{\sc i\kern-.025em b}\kern-.08em
    T\kern-.1667em\lower.7ex\hbox{E}\kern-.125emX}}

\usepackage{algorithm}
\usepackage{array}
\usepackage[caption=false,font=normalsize,labelfont=sf,textfont=sf]{subfig}
\usepackage{textcomp}
\usepackage{stfloats}
\usepackage{url}
\usepackage{verbatim}
\usepackage{graphicx}

\usepackage{booktabs}
\usepackage{multirow}
\usepackage{algpseudocode}
\usepackage{adjustbox}
\usepackage{hyperref}
\usepackage{orcidlink}

\hypersetup{
    pdftitle = {SDOoop: Capturing Periodical Patterns and Out-of-phase Anomalies in Streaming Data Analysis},
    pdfauthor = {Alexander Hartl, Félix Iglesias Vázquez and Tanja Zseby}
}

\newtheorem{theorem}{Theorem}
\newtheorem{definition}{Definition}
\newtheorem{lemma}{Lemma}
\newtheorem{proof}{Proof}

\begin{document}
\title{SDOoop: Capturing Periodical Patterns and Out-of-phase Anomalies in Streaming Data Analysis}

\author{\IEEEauthorblockN{Alexander Hartl \orcidlink{0000-0003-4376-9605}}
\IEEEauthorblockA{\textit{Institute of Telecommunications} \\
\textit{TU Wien}\\
1040 Wien, Austria\\
me@alexhartl.eu}
\and
\IEEEauthorblockN{Félix Iglesias Vázquez \orcidlink{0000-0001-6081-969X}}
\IEEEauthorblockA{\textit{Institute of Telecommunications} \\
\textit{TU Wien}\\
1040 Wien, Austria\\
felix.iglesias@tuwien.ac.at}
\and
\IEEEauthorblockN{Tanja Zseby \orcidlink{0000-0002-5391-467X}}
\IEEEauthorblockA{\textit{Institute of Telecommunications} \\
\textit{TU Wien}\\
1040 Wien, Austria\\
tanja.zseby@tuwien.ac.at}
}

\maketitle

\begin{abstract}
Streaming data analysis is increasingly required in applications, e.g., IoT, cybersecurity, robotics, mechatronics or cyber-physical systems. 
Despite its relevance, it is still an emerging field with open challenges. 
SDO is a recent anomaly detection method designed to 
meet requirements of speed, interpretability and intuitive parameterization. In this work, we present SDOoop, which extends the capabilities of SDO's streaming version to 
retain temporal information of data structures. SDOoop spots contextual anomalies undetectable by traditional algorithms, while enabling the inspection of data geometries, clusters and temporal patterns. 
We used SDOoop to model real network communications in critical infrastructures and extract patterns that disclose their dynamics. 
Moreover, we evaluated SDOoop with data from intrusion detection and natural science domains and obtained performances equivalent or superior to state-of-the-art approaches. 
Our results show the high potential of new model-based methods to analyze and explain streaming data.
Since SDOoop operates with constant per-sample space and time complexity, it is ideal for big data, being able to instantly process large volumes of information. 
SDOoop conforms to next-generation machine learning, which, in addition to accuracy and speed, is expected to provide highly interpretable and informative models.  
\end{abstract}


\section{Introduction}
In data stream processing, data points $\boldsymbol v_j \in \mathbb R ^D$ consistently arrive at monotonically increasing times $t_j \in \mathbb R$ for $j=1,2,\ldots$ 
Due to this steady acquisition,
analysis algorithms face the challenge of discovering knowledge in unbounded data that substantially accumulates in a short time.
In such a context, real-life applications dismiss batch-mode operation while demanding fast online processing able to update models and parameters to \emph{concept drift}.
Here, ``updating models and parameters'' does not only mean adapting to new patterns and classes, but also forgetting those that have become obsolete.

In anomaly/outlier detection (OD), 
we commonly set a sliding window (or an observation horizon) $w$ that establishes the memory length for which space geometries are remembered. 
Hence, the anomaly is defined: (a) either based on the neighborhood of a data point within $w$ (e.g., Exact- and Approx-STORM~\cite{Angiulli2007}), or (b) by comparing to a model that evolves with $w$ (e.g., Robust Random Cut Forests \cite{Guha2016}). 
In both cases, note that the comparison reference is purely static (or geometric) relative to the point of comparison at the instant of comparison. That is, the data within $w$ (or the model used instead) is a snapshot.  
While this is not a problem for many types of anomalies, most traditional methods are blind to identify contextual anomalies. A contextual (aka. conditional or out-of-phase) anomaly ``occurs if a point deviates in its local context'' \cite{Ruff2021}, i.e., if it happens outside its usual time. Consider a method whose $w$ spans a one-week period. If a cluster occurs exclusively during weekends, but a data point of this cluster accidentally appears on a Wednesday, this method \emph{will not} identify it as an anomaly, but as a normal inlier instead. 

\begin{figure}[b]
	\centering
	\includegraphics[width=0.9\columnwidth]{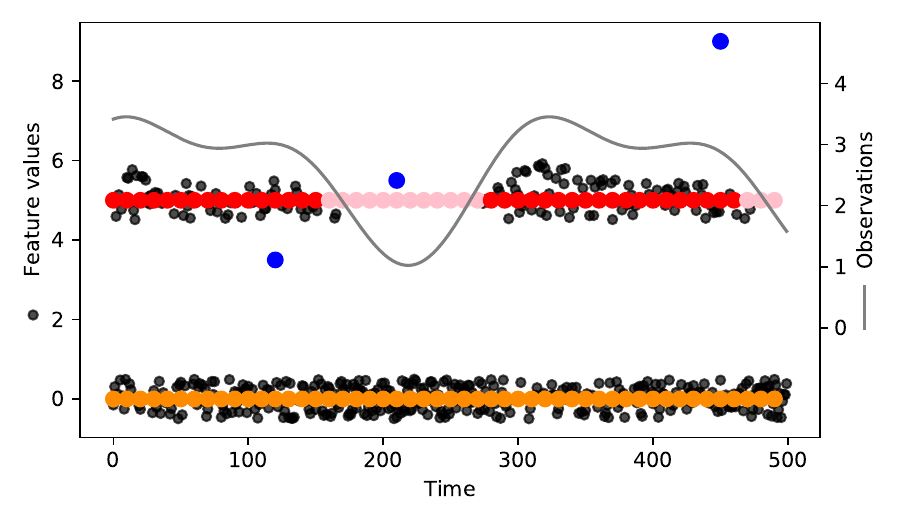}
	\caption{Example of a data stream, a model with two observers ({red} and {orange}), and three types of anomalies~({blue}): local (left), contextual (middle), global (right).}
	\label{fig:obs}
\end{figure}

Here, we present SDOoop 
(SDO out-of-phase), an algorithm for OD in streaming data whose models 
store temporal information.  
While retaining constant per-sample space and time complexity and keeping intact the functionalities to detect other types of anomalies,
SDOoop is also able to identify contextual anomalies and capture periodical patterns that explain the time behavior of the data bulk.
SDOoop builds models 
by sampling a fixed number of data points at representative locations in feature space, called \emph{observers}.
To escape dependence of the data volume, it uses an exponentially weighted moving average (EWMA) to estimate model information from the arriving data mass. At the same time, observers  hold temporal information as coefficients of Fourier transforms (FT). 
Thus, for a specific time of interest $t$, observers ``twinkle'' to show only the most representative model for time $t$.
The simple example in Fig. \ref{fig:obs} can give an intuition of the distinctive capabilities of SDOoop when compared with alternatives. In the figure, the internal model (incrementally updated) is formed by the red and orange points, which represent observers. 
Observers are placed in areas of considerable density to see the data mass around. Note that, while the orange observer stands for a cluster that occurs at a continuous pace, the red observer represents a cluster that exhibits a temporal behavior. Consequently, the red observer \emph{twinkles} accordingly, the drawn gray curve showing the inverse FT of its captured FT coefficients.  
If a contextual anomaly happens, closest observers will not be awake, hence it will be detected as an anomaly by far observers.

Our work advances the research on observers-based unsupervised learning, which originated SDO~\cite{Iglesias2018}, SDOstream~\cite{Hartl2019} and SDOclust~\cite{Iglesias2023}.  
The remainder of this paper is structured as follows:
In Section \ref{sec:preliminaries}, we introduce observers-based OD.
Section \ref{sec:method} describes SDOoop and explains its parameters. 
SDOoop is evaluated in Section \ref{sec:experiments}. 
In Section~\ref{sec:related_work} we explore related research efforts and contrast them to the problem solved here.
Finally, we summarize the main ideas and contributions in Section \ref{sec:conclusion}. 
To enable reproducibility, we make all our source code available in our repository 
\url{https://github.com/CN-TU/tpsdos-experiments}.

\section{Observers-based Outlier Detection} \label{sec:preliminaries}
The Sparse Data Observers~(SDO) method~\cite{Iglesias2018} for OD on static datasets is the foundation of our current proposal. 
In a nutshell, SDO works as follows:
(a) Randomly sample points from the dataset, which will be called ``observers''.
(b) Each data point in the dataset is \emph{only} observed by the $x$ nearest observers, resulting in each observer performing a different number of observations. 
(c) Remove \emph{idle} observers, i.e., with the smallest number of observations. Remaining observers are termed \emph{active}. 
(d) For each data point, compute an outlier score as median distance to the $x$ nearest \emph{active} observers.
Hence, observers capture the main shapes of the data in a low-density model and outlier scores are calculated as distances to points in this model. Removing idle observers minimizes the chances of outliers being part of the model.  
The observer-approach also holds for SDOstream~\cite{Hartl2019}, which adapts the algorithm for a streaming setting by continuously sampling new observers and using EWMA for computing observations.

In this paper, we predominantly adhere to the notation of~\cite{Hartl2019}. Hence, we denote the observers set as $\Omega$ and, accepting a slight abuse of notation, we denote by $\boldsymbol{\omega} \in \Omega$ both an abstract observer and its feature vector. Furthermore, $P_{\boldsymbol{\omega}}$ denotes $\boldsymbol{\omega}$'s observations, where $P_{\boldsymbol \omega} \in \mathbb N_0$ for SDO and $P_{\boldsymbol \omega} \in \mathbb R_0^+$ for SDOstream. Hence, $P_{\boldsymbol{\omega}}$  counts the number of data points for which $\boldsymbol{\omega}$ belongs to the $x$ nearest observers with an algorithm parameter $x \in \mathbb N$. Observers with insufficient $P_{\boldsymbol{\omega}}$ are thus disregarded for outlier scoring.
In contrast to SDO and SDOstream, SDOoop replaces the number of observations $P_{\boldsymbol{\omega}}$ with a temporal function, allowing active observers to become temporarily idle (i.e., \emph{asleep}) and reappear dynamically in accordance with the temporal pattern of the underlying clusters. Therefore, it is possible to construct an active observers set representative for the data stream at the current time and, hence,
to detect data points that do not meet the established temporal pattern, i.e., contextual anomalies.

\section{SDOoop} \label{sec:method}
\label{sec:oop}

\begin{table}[t]
	
	\caption{Symbols and notation}
	\label{tab:notation}
	\vspace*{-1mm}\rule{\columnwidth}{\heavyrulewidth}
	\begin{tabular}{l r@{\hspace{1mm}}l@{\hspace{2mm}}l}  \\[-2mm]
		\multirow{6}{*}{\rotatebox{90}{\parbox[c]{1.5cm}{\centering Algorithm Parameters}}} & $k$ & $\in \mathbb N$ & Number of observers. \\
		& $x$ & $\in \mathbb N$ & Number of nearest observers.\\
		& $T$ & $\in \mathbb R^+$ & EWMA time constant. \\
		& $T_0$ & $\in \mathbb R^+$ & FT base period.\\
		& $N_{\text{bins}}$ & $\in \mathbb N$ & Number of frequency bins.\\
		& $q_{\text{id}}$ & $\in [0,1]$ & Observer idle-active fraction.\\[3mm]
		
		\multirow{4}{*}{\rotatebox{90}{\parbox[c]{1.5cm}{\centering Algorithm State}}} & $\Omega$ &  & Observers set.\\
		& $P_{\boldsymbol{\omega},n}$ & $\in \mathbb C$ & Fourier coefficients for $\boldsymbol{\omega}$'s observations.\\
		& $H_{\boldsymbol{\omega}}$ & $\in \mathbb R^+$ & Reference for age-normalization of $P_{\boldsymbol{\omega},0}$. \\
		& $i_{\text{LAO}}$ & $\in \mathbb N$ & Index of last added observer.\\[3mm]
		
		\multirow{5}{*}{\rotatebox{90}{\parbox[c]{2.5cm}{\centering Further Notation}}} & $\boldsymbol{\omega}$ & $\in \Omega$ & An observer.\\
		& $d$ & $(\cdot,\cdot)$  & A distance function. \\
		& $\mathcal N$ & $\subset \Omega$ & Set of nearest observers. \\
		& $\mathcal N_a$ & $\subset \Omega$ & Set of nearest active observers. \\
		& $n$ & $\in [N_{\text{bins}}]$ & The frequency index. \\
		& \multicolumn{2}{c}{$\in_R$ } & Uniformly random sampling from a given set.\\
		& \multicolumn{2}{c}{$\mathbb R, \mathbb C$ } & Sets of real and complex numbers, respectively.\\
		& \multicolumn{2}{c}{$\Re(\cdot),\Im(\cdot)$ } &  Real and imaginary part, respectively.\\[1mm]
	\end{tabular}
	\rule{\columnwidth}{\heavyrulewidth}
\end{table}

We describe the construction of our proposal. 
Main symbols and notation are shown in Table~\ref{tab:notation}. 
We denote by $[ N_{\text{bins}} ]$ with $N_{\text{bins}} \in \mathbb N $ the set $\{0,\ldots,N_{\text{bins}}-1\}$ and by $d: \mathbb R ^{D} \times \mathbb R^{D} \rightarrow \mathbb R^{+} $ a distance function (e.g., Euclidean). 
Our method enables the model to absorb temporal patterns in processed data streams. To describe this, we consider data streams satisfying the following definition.

\begin{definition}
	\label{temp_definition}
	For a given data stream, let $\gamma(\boldsymbol v, t) \in \mathbb R_0^+$ denote the expected rate of arriving data points at location $\boldsymbol v \in \mathbb R^D$ and time $t \in \mathbb R$. Therefore, $\gamma(\boldsymbol v, t)\Delta \boldsymbol v \Delta t$ stands for the expected number of data points seen in a volume $\Delta \boldsymbol v$ and time interval $\Delta t$. We say that the stream exhibits $T_0$-periodic patterns with $T_0 \in \mathbb R^+$ if $\gamma(\boldsymbol v,t)$ is $T_0$-periodic, i.e., $\gamma(\boldsymbol v,t)=\gamma(\boldsymbol v,t+T_0) $ for all $ \boldsymbol v \in \mathbb R^D, t \in \mathbb R$.
\end{definition}

Definition~\ref{temp_definition} is based on the \emph{expected} rate of arriving data points. This means that, to reason about periodic behaviors, the random stream is modeled as generated by an underlying deterministic process.
In particular, a stationary stream exhibits $T_0$-periodic patterns for any $T_0$. 
Note that Definition \ref{temp_definition} does not include concept drift, which is tackled by SDOoop with a exponential sliding window.

To capture temporal patterns, we allow the observers' observations to be $T_0$-periodic. We represent and store the associated temporal functions in terms of their FT coefficients $P_{\boldsymbol{\omega},n} \in \mathbb C$ with $\boldsymbol{\omega} \in \Omega, n \in [N_{\text{bins}}]$.
To extract observers relevant for the current point in time from the model, we first define the $q_{\text{id}}$-percentile $P_{\text{thr}} \in \mathbb R^+$ of the observers' average observations $P_{\boldsymbol \omega,0}$, i.e.,
\begin{equation}
P_{\text{thr}} = \max \Big\{ \rho \in \mathbb R^+ \, \Big\vert \, \big\vert\{ \boldsymbol{\omega} \in \Omega \, \vert \,  P_{\boldsymbol{\omega},0} < \rho \}\big\vert \le q_\text{id} \vert\Omega\vert \Big\}.
\end{equation}
Similar to previous work, $P_{\text{thr}}$ allows us to require active observers to have a minimum number of observations in relation to the total time-averaged observation count.
Hence, we construct a view yielding the currently \emph{active} observers 
\begin{equation} \label{eq:omega_a}
\Omega_a = \Bigg\{ \boldsymbol{\omega} \in \Omega \, \Big\vert \, \Re \left\{ \sum{}_{n \in [N_{\text{bins}}]} P_{\boldsymbol{\omega},n} \right\} \ge P_{\text{thr}} \Bigg\} 
\end{equation}
in terms of a lower-bound for the inverse FT, where $\Re \left\{ \sum{}_{n \in [N_{\text{bins}}]} P_{\boldsymbol{\omega},n} \right\}$ evaluates the temporal shape of the observers' observations at the current time.

To narrow the scope to the most relevant information, we form sets from both $\Omega$ and $\Omega_a$ that only contain the $x$ nearest points. Hence, for a point  $\boldsymbol v \in \mathbb R^D$ we specify the set of nearest observers $\mathcal N(\boldsymbol v)\subset \Omega$ with $ \vert\mathcal N \vert  = \min(x,\vert\Omega\vert)$ and the set of nearest \emph{active} observers $\mathcal N_a(\boldsymbol v) \subset \Omega_a$ with $ \vert \mathcal N_a \vert = \min(x,\vert\Omega_a\vert)$, i.e.
\begin{align}
d( \tilde {\boldsymbol{\omega}}, \boldsymbol v ) &\le d( \boldsymbol{\omega}, \boldsymbol v ) \, \forall \, \tilde {\boldsymbol{\omega}} \in \mathcal N(\boldsymbol v), \boldsymbol{\omega} \in \Omega \setminus \mathcal N(\boldsymbol v) \, \text{ and} \\
d( \tilde {\boldsymbol{\omega}}, \boldsymbol v ) &\le d( \boldsymbol{\omega}, \boldsymbol v ) \, \forall \, \tilde {\boldsymbol{\omega}} \in \mathcal N_a(\boldsymbol v), \boldsymbol{\omega} \in \Omega_a \setminus \mathcal N_a(\boldsymbol v) .
\end{align}
Algorithm \ref{alg:sdostream} depicts the core process, discussed as follows.

\begin{algorithm}[thb]
	\caption{Processing a data point $(\boldsymbol v_i,t_i)$.}
	\begin{algorithmic}[1]
		\State
		Find $x$ nearest observer sets $\mathcal N(\boldsymbol v_i)$ and $\mathcal N_a(\boldsymbol v_i)$
		\State \textbf{report} median$_{\boldsymbol{\omega} \in \mathcal N_a} d( \boldsymbol{\omega}, \boldsymbol v_i $) as outlier score
		\State Set $H_{\boldsymbol{\omega}} \leftarrow H_{\boldsymbol{\omega}} \exp(-\frac{t_i-t_{i-1}}{T}) + 1 \quad \forall \, \boldsymbol{\omega} \in \Omega$
		\State Set $P_{\boldsymbol{\omega},n} \hspace*{-1mm}  \leftarrow \hspace*{-1mm} P_{\boldsymbol{\omega},n} \hspace*{-1mm} \left[\exp(\text{-}\frac{1}{T}\text{+} \frac{jn2\pi}{T_0})\right]^{t_i\text{-}t_{i\text{-}1}} \hspace*{-1mm} \forall \boldsymbol{\omega}  \in  \Omega, n \in [N_{\text{bins}}]$
		\State Set $P_{\boldsymbol{\omega},n} \leftarrow P_{\boldsymbol{\omega},n} + 1  \forall \, \boldsymbol{\omega} \in \mathcal{N}, n \in [N_{\text{bins}}]$
		\If{$\vert \Omega \vert  =  0$ \textbf{or} $r \le \frac{k^2}{Tx} \frac{\sum_{\boldsymbol{\omega} \in \mathcal N} P_{\boldsymbol{\omega},0}}{\sum_{\boldsymbol{\omega} \in \Omega} P_{\boldsymbol{\omega},0}} \frac{t_i\text{-}t_{i_{\text{LAO}}}}{i\text{-}i_{\text{LAO}}}$ with $r  \in_R [0,1]$}
		\State Remove arg\,min$_{\boldsymbol{\omega} \in \Omega} \, \frac{P_{\boldsymbol{\omega},0}}{H_{\boldsymbol{\omega}}}$ from $\Omega$ \textbf{if} $\vert \Omega \vert = k$
		\State Add $\boldsymbol v_i$ to $\Omega$
		\State {Set $i_{\text{LAO}} \leftarrow i, \, H_{\boldsymbol v_i} \leftarrow 1$ and $P_{\boldsymbol v_i,n} \leftarrow 1 \forall n \in [N_{\text{bins}}]$}
		\EndIf
	\end{algorithmic}
	\label{alg:sdostream}
\end{algorithm}

\subsection{Algorithm Construction}
Algorithm \ref{alg:sdostream} can be divided into three parts: establishing active observers $\Omega_a$ (line 1), scoring outlierness (line 2), and updating the model (lines 3-10). 
The core concept, which allows to capture periodical patterns, is based on Lemma \ref{ft_lemma}.

\begin{lemma}
	\label{ft_lemma} 
	For an observer $\boldsymbol{\omega} \in \Omega$, let $g(t)\in \mathbb R^+$ denote the expected rate of arriving data points, for which $\boldsymbol{\omega}$ is contained in $\mathcal N$ at time $t$. If $g(t)$ is a $T_0$-periodic function and $T\gg T_0$, observations $P_{\boldsymbol{\omega},n}$ approximate a Fourier transform $E\{P_{\boldsymbol{\omega},n}\} \approx \int _{-T_0}^{0} g(\tau-t) \exp(-j2 \pi n\tau/T_0) d\tau$ up to a constant factor.
\end{lemma}

We prove the lemma in Appendix~\ref{sec:proof1}.
Lemma \ref{ft_lemma} shows that temporal information about how frequently observers are used can be extracted from $P_{\boldsymbol{\omega} , n}$ in terms of an inverse FT. 
To obtain the current set of active observers $\Omega_a$, 
it suffices by selecting observers that have been used most often in the past.
Here, 
observer activity is mainly evaluated at time $t-T_0$, which is reasonable due to $T_0$-periodicity. However, due to inherent interpolation, also the very recent activity of observers is considered, which is particularly relevant for setting new observers.
In Theorem \ref{omega_a_theorem}, we show that our method applies this approach for constructing $\Omega_a$.

\begin{theorem}
	\label{omega_a_theorem}
	At time $t$, for data streams with $T_0$-periodic patterns, the active observers set $\Omega_a$, as used by Algorithm \ref{alg:sdostream}, contains observers with highest $g(t).$
\end{theorem}
\begin{proof}
	Equation~\ref{eq:omega_a} constructs the set $\Omega_a$ by selecting observers from $\Omega$, for which $\Re \{\sum_{n\in[N_{\text{bins}}]}P_{\boldsymbol{\omega},n}\}$ is highest. If $P_{\boldsymbol{\omega},n}$ yields the FT of $g(t)$ according to Lemma \ref{ft_lemma}, the theorem follows immediately, since $\Re\{\sum_{n\in[N_{\text{bins}}]}P_{\boldsymbol{\omega},n}\}$ performs the inverse FT at time $t=0$ relative to the current time.
\end{proof}

Theorem \ref{omega_a_theorem} allows us to use $\Omega_a$ for assessing outlierness of arriving data points by leveraging nearest-observer distances. Hence, in line 1, $\mathcal N$ and $\mathcal N_a$ are constructed. Based on \cite{Iglesias2019}, we compute an outlier score with the median of distances to the $x$ closest observers. 
The final part of Algorithm \ref{alg:sdostream} handles model updating, which involves replacing the less ``active'' observer. 
Updating of $P_{\boldsymbol{\omega}, n}$ in line 4 follows an exponential shape set by time $T$. We show in Theorem \ref{sample_count}  that replacing observers proceeds with the same pace, which is necessary as observers otherwise would not be able to build meaningful  $P_{\boldsymbol{\omega}, n}$ values.
\begin{theorem}
	\label{sample_count}
	For data streams exhibiting $T_0$-periodic patterns, Algorithm \ref{alg:sdostream} on average samples $k$ data points during a time period $T$ as new observers.
\end{theorem}

We prove the lemma in Appendix~\ref{sec:proof2}.
Note that the factor $\frac{\sum_{\boldsymbol{\omega} \in \mathcal N } P_{\boldsymbol{\omega},0} /x}{\sum_{\boldsymbol{\omega} \in \Omega } P_{\boldsymbol{\omega},0} /k}$ occurring in line 6 of Algorithm~\ref{alg:sdostream} might be omitted without invalidating Theorem \ref{sample_count}. 
However, we include it to promote representativity of the observers set. 
Hence, underrepresented observers in a  neighborhood cause the observation count in this neighborhood to increase, leading to a higher sampling probability, while overrepresented observers lead to a lower sampling probability. 
Moreover, this factor is stronger during the transient starting phase, ensuring that the model soon reaches its full size, but at the same time avoiding that it is fulfilled with the first points, which in many cases would build unrepresentative models.

During time $T$ the model is replaced one time on average according to Theorem \ref{sample_count}. Since we use a fixed-size model, an observer has to be removed when adding a new one, picking the removal candidate based on its $P_{\boldsymbol{\omega}, n}$. 
To avoid new observers constantly replaced due to the stronger inertia of old observers, we use an age-normalized observation count $\frac{P_{\boldsymbol {\omega},0}}{H_{\boldsymbol {\omega}}}$ for selecting the observer to remove in line 7. By updating $H_{\boldsymbol{\omega}}$ as depicted in line 3,  $H_{\boldsymbol{\omega}}$ denotes the maximum $P_{\boldsymbol{\omega},0}$ that an observer $\boldsymbol{\omega}$ might have reached over time. Thus, $\frac{P_{\boldsymbol {\omega},0}}{H_{\boldsymbol {\omega}}} \in [0,1]$, where 1 is only scored if $\boldsymbol {\omega}$ has always been in $\mathcal N$ since it was assimilated in the model.

\subsection{Interpreting the Learned Model}

Direct analysis of patterns in streams from a manual perspective is inherently complicated. Experts quickly run into difficulties regarding how to observe the data, what reference points to take, for how long, how much data to use, how to do this incrementally, etc. 
SDOoop solves all these issues in a natural and elegant way.
At any point in time, the observers set $\Omega$ allows prompt access to highly representative data points that additionally retain temporal information. The temporal shape $g_{\boldsymbol{\omega}}(t)$ of observed data points in a neighborhood of a given observer $\boldsymbol \omega$ can be efficiently recovered in terms of an inverse FT,
\begin{equation} \label{eq:ift_experiments}
g_{\boldsymbol{\omega}}(t) = \Re \left\{ \sum{}_{n \in [N_{\text{bins}}]} P_{\boldsymbol{\omega},n} \exp(jtn2\pi /T_0) \right\} .
\end{equation}

SDOoop is designed to be embedded in the data processing pipeline of stand-alone systems. In other words, its primary purpose is feeding subsequent analysis phases, e.g., visualization or clustering techniques to extract further knowledge. Nevertheless, $\Omega$ is commonly small enough for manual inspection, 
meaning that we can explore the set of observers along different time spans and study their periodicities, but also isolate a given instant in time to focus only on its stream characteristics.  
By analyzing observers in $\Omega_a$, the data analyst obtains an immediate depiction of the data model to easily interpret both the ground of the outlierness scoring and how data are (or are expected to be) as a whole.

\subsection{Costs and Parameters}

The computational cost lies primarily in the comparison of incoming data points with the learned model.
Assuming that distance computation is $\mathcal{O}(D)$ with the number of dimensions $D$, 
building $\Omega_a$ (equation \ref{eq:omega_a}) implies 
time complexity $\mathcal{O}(kN_{\text{bins}})+\mathcal{O}(kD)$. 
Holding the model in memory requires storing the observers $\boldsymbol\omega$ and storing their observations $P_{\boldsymbol \omega,n}$, similarly resulting $\mathcal{O}(kN_{\text{bins}})+\mathcal{O}(kD)$ as space complexity.
Therefore, per-point space and time complexity linearly depend on model size $k$, which is a pre-fixed parameter.  
This makes SDOoop suitable for big data with highly demanding processing.

\begin{figure*}[b]
	\centering
	\includegraphics[width=0.9\textwidth]{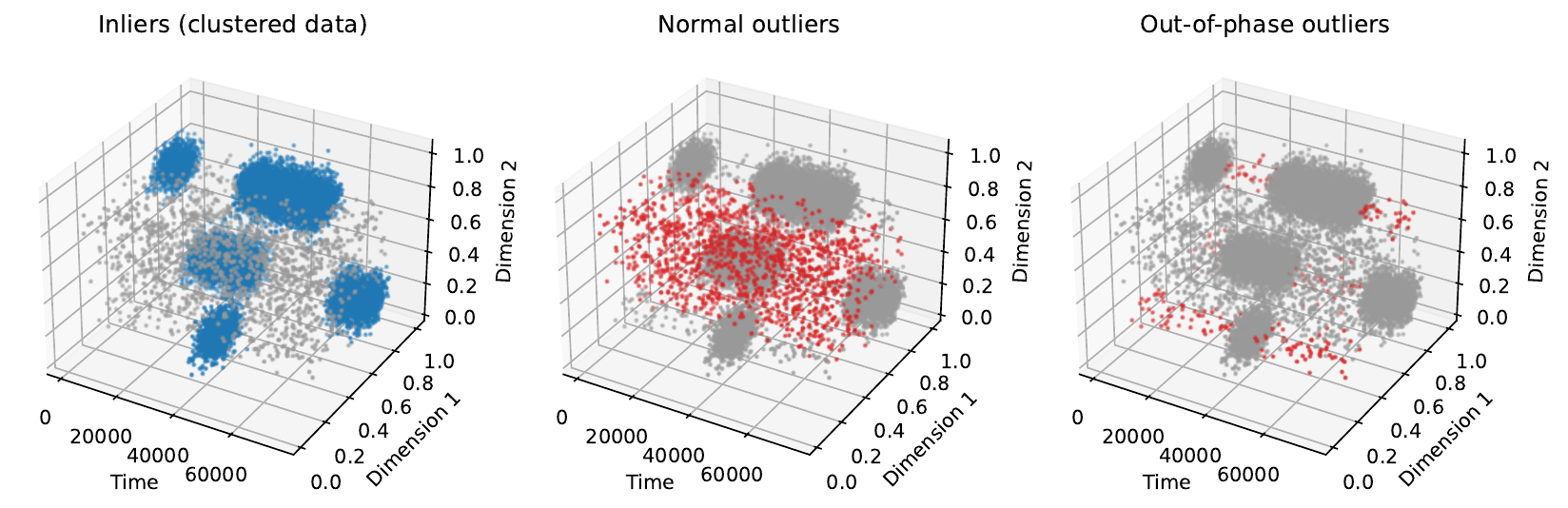}
	\caption{Normal outliers and contextual outliers (aka out-of-phase outliers) in synthetic data for a fraction of contextual outliers of 0.5\%.}
	\label{fig:pocdata}
\end{figure*}

Temporal behavior is captured by 
$T$, $T_0$ and the number of frequency bins $N_{\text{bins}}$. 
$T$ is the time constant of the exponential windowing mechanism.  
It governs memory length, therefore equivalent to the window length of sliding window algorithms.
$T_0$ denotes the period of the FT base frequency. Periodicities can be captured best if $T_0$ is an integer multiple of expected periodicities.
For instance, in many real-world applications, it might be reasonable a $T_0$-value of one week, so that weekly and diurnal patterns can be detected. 
Furthermore, to ensure that the EWMA approximates a Fourier integral, $T_0$ should be reasonably smaller than $T$.
$N_{\text{bins}}$ determines the maximum frequency that can be captured by the model,
hence also fixing the temporal resolution of the learned temporal shapes.

$T$, $T_0$ and $N_{\text{bins}}$ are intuitive parameters and can be easily adjusted based on domain knowledge. 
$q_{\text{id}}$, $k$ and $x$ are discussed extensively in \cite{Iglesias2018} and \cite{Hartl2019}.
Here suffice it to mention that further experimentation confirms $q_{\text{id}}$, $k$ and $x$ robustness, meaning that performances are stable for a wide range of values and that most applications work properly with default configurations.
The setting of $k$ depends on the expected variability and degree of representation, but several hundred observers is sufficient in most cases. 
$x$ inherits from nearest-neighbor algorithms, with similar tuning strategies \cite{Hall2008}. 
In our experiments, values between $k \in [100,1000]$, $x \in [3,9]$ and $q_{\text{id}} \in [0.1,0.3]$ have shown excellent results.

\section{Experimental Evaluation} \label{sec:experiments}
In this section, we discuss the experimental evaluation of SDOoop.
Based on a proof of concept, we first demonstrate its capability to detect contextual outliers. We then proceed by benchmarking OD performances on public datasets. We finally show the discovery and modeling of temporal patterns in real-life cases.

\subsection{Contextual Anomalies: Proof of Concept}
For this proof of concept, we use MDCGen~\cite{Iglesias2019} to generate a synthetic data stream of five clusters that vanish and reappear at different times and periods. We add both spatial and contextual outliers into the stream. 
Fig.~\ref{fig:pocdata} shows an excerpt of the generated data stream with 0.5\% of contextual outliers. Hence, while normal outliers are distributed across the entire feature space, outliers occurring out of phase fall in the same spatial location as clustered data, but their time of appearance does not meet the temporal shape of clustered data.

\begin{figure}[t]
	\centering
	\includegraphics[width=0.9\columnwidth]{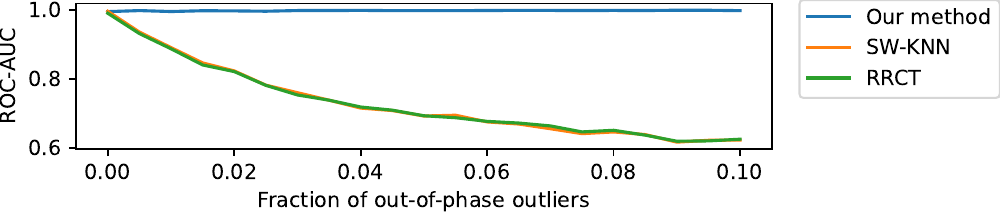}
	\caption{OD performance vs contextual outlier rate in the proof of concept.}
	\label{fig:poc}
\end{figure}

In Fig.~\ref{fig:poc}, we plot the area under the ROC curve (AUC) for different ratios of contextual outliers to data points in one active cluster.
We compare SDOoop with two consolidated OD methods for streaming data: 
SW-$k$NN (the sliding-window implementation of $k$NN for OD~\cite{Ramaswamy2000}) and RRCT~\cite{Guha2016}. 
All algorithms have been properly tuned to capture at least one full period. The more outliers occur out of phase, the more the performance of traditional algorithms plummets, 
whereas our method retains the highest AUC at all times.
This clearly indicates that 
SDOoop is the only algorithm capable of detecting contextual outliers.

\subsection{OD Performance with Evaluation Datasets}
To compare our method with state-of-the-art stream OD algorithms, we selected popular OD datasets of sufficient length 
and with timestamped data points.

\textbf{Datasets and metrics.} 
The KDD Cup'99 dataset \cite{Kddcup99} aims at detecting network intrusions based on a number of network and host features and, similar to previous work~\cite{Campos2016},
we considered User to Root (U2R) attacks as outliers over normal traffic, resulting in 976,414 data points with an outlier proportion of 0.4\%.
Additionally, we selected the recent SWAN-SF \cite{Angryk2020} dataset, which collects data about solar flares, and used preprocessing scripts provided by Ahmadzadeh and Aydin~\cite{Ahmadzadeh2020}. For SWAN-SF, we assigned a normal label to the majority class and an outlier label to the remaining classes, resulting in 331,185 data points with an outlier portion of 17.2\%. In both experiments, we randomly sampled 50\% of the data stream for randomized hyperparameter search and the other half for evaluation. For an overview of the ranges of hyperparameters, we refer to the code repository of this paper. Metrics for evaluation are Adjusted Average Precision (AAP), 
Adjusted Precision at $n$ (AP@n), 
and AUC~\cite{Campos2016}.

\textbf{Algorithms and experimental setups.} 
We used the dSalmon framework~\cite{Hartl2024}, which provides efficient versions of several stream OD algorithms. 
Since ensembles commonly exhibit superior accuracy, we used an ensemble of nine for SDOoop, yet noticing almost no difference compared to a single SDOoop detector. 

\textbf{OD performances.} 
Experiment results in Table~\ref{tab:performance_results} show how our method matches and even outperforms state-of-the-art algorithms for streaming OD. The strongest competitor is RS-Hash \cite{Sathe2016}. 
In the SWAN-SF case, SDOoop ranks among the best performers, while, in the KDD Cup'99 dataset, it clearly stands out, particularly in AAP and AP@n. 
The higher AP@n also indicates that our method finds several true outliers that pass unperceived for the competitors.

\begin{table}[t]
	\centering
	\caption{Performance comparison with different OD algorithms}
	\label{tab:performance_results}
	\vspace*{-1mm}\rule{\columnwidth}{\heavyrulewidth}
	\begin{tabular}{l r r r r r r} \\[-2mm]
		& \multicolumn{3}{c}{\textbf{SWAN-SF} \cite{Angryk2020}} & \multicolumn{3}{c}{\textbf{KDD Cup'99} \cite{Kddcup99}}\\
		& AAP & AP@n & AUC & \hspace{2mm}AAP & AP@n & AUC\\
		SW-$k$NN & 0.69 & \textbf{0.56} & \textbf{0.91} & 0.07 & 0.15 & 0.72\\
		SW-LOF & 0.15 & 0.12 & 0.58 & -0.00 & -0.00 & 0.67\\
		LODA~\cite{Pevny2016} & 0.72 & 0.54 & \textbf{0.91} & 0.10 & 0.13 & 0.92\\
		RS-Hash~\cite{Sathe2016} & \textbf{0.73} & 0.55 & \textbf{0.91} & 0.13 & 0.15 & 0.95\\
		RRCT~\cite{Guha2016} & 0.23 & 0.19 & 0.69 & 0.07 & 0.05 & 0.85\\
		SDOoop & \textbf{0.73} & 0.55 & \textbf{0.91} & \textbf{0.33} & \textbf{0.54} & \textbf{0.97}
		\\[1mm]
	\end{tabular}
	\rule{\columnwidth}{\heavyrulewidth}
\end{table}

\textbf{Disclosing insights about the data.} 
Obtained results seem consistent with data contexts. Considering how
outliers have been defined in the SWAN-SF dataset, we do not expect
that outliers break possible temporal periodicities in samples from solar
flares. 
On the other hand, the KDD Cup’99 dataset describes events in a
computer network, which are expected to exhibit strong
temporal patterns due to human activity. Patterns may be broken by attack traffic, 
leading to contextual outliers, behaviors that can be spotted by our method. 
Here, the superior detection of
SDOoop not only indicates that the data
show temporal patterns, but also that some U2R
attacks are indeed contextual outliers.

\subsection{Temporal Patterns in Machine-to-Machine Communication} \label{sec:experiments_m2m}

\textbf{Application context.} 
We study network traffic captured in a critical infrastructure,
in particular, of an energy supply company that connects charging stations for electric vehicles. The network communication satisfies management, accounting and maintenance aspects\footnote{While we embrace reproducible research, issues related to confidentiality, security and privacy prevent us from making these data publicly available.}.
Network communication for these purposes usually adopts the OCPP protocol. 
Due to the large portion of machine-to-machine communications, we expected to discover distinct periodic patterns.

\textbf{Preprocessing and parameters.} 
We preprocessed data with the feature vector described in \cite{Williams2006}, resulting in 13 million flows during a 1 month period. We parameterized the algorithm using $T=1$ week, 2000 frequency bins and $T_0=2000$ minutes, obtaining a minimum period of 1 minute. 
We used 400 observers. 

\textbf{Capturing periodical patterns/clusters.} 
Fig.~\ref{fig:vkw} shows on the left side examples for the frequency spectrum (magnitude) learned by observers. Hence, different clusters show diverse temporal patterns. While observer 1 shows no or just weak periodicities, observer 2 shows a clear 5 minute periodicity and observers 3 and 4 show a 10 minute periodicity. From the learned FT, temporal shapes can be constructed in terms of an inverse FT as depicted in equation~\ref{eq:ift_experiments}.
Fig.~\ref{fig:vkw} also shows the reconstructed temporal shape plotted over a 1 hour and 24 hour period. Hence, beneath the periodicities already found when inspecting the FT directly, the temporal shape for observers 3 and 4 additionally shows periodicities of a longer period of approximately 2.5 hours.

\begin{figure*}[htb!]
	\centering
\includegraphics[width=0.28\textwidth]{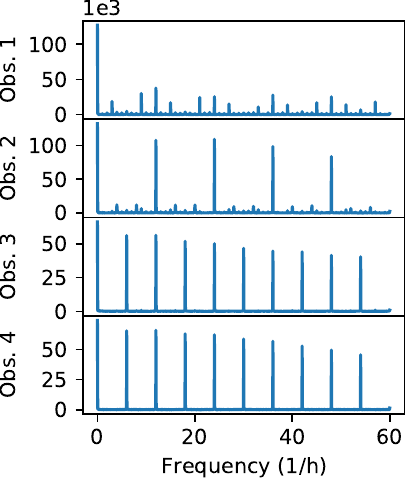} \quad
\includegraphics[width=0.28\textwidth]{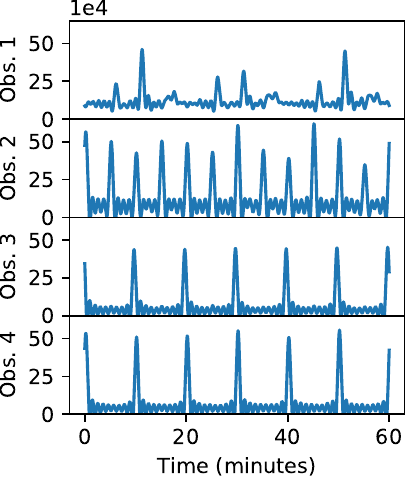} \quad
\includegraphics[width=0.28\textwidth]{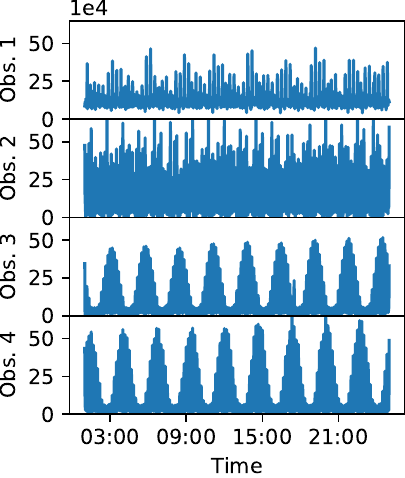}
	\caption{Learned magnitude spectrum (left), one-hour temporal plots (middle) and 24-hour temporal plots (right) for four exemplary observers when processing network data captured in an e-charging infrastructure.}
	\label{fig:vkw}
\end{figure*}

\textbf{Interpreting clusters in the application.} 
The manual examination of network flows represented by observers confirmed the soundness of discovered temporal patterns. For example, observer 3 corresponds to ICMP pings that happen regularly to ensure that network devices are alive. Observer 4 identifies DNS requests that charging stations perform to resolve the name of the OCPP server to its IP address and transmit meter readings. For observer 4, the periodicity emerges from DNS caching, so that every second request for transmitting meter readings can be performed without having to perform a DNS lookup.

Observer 1 corresponds to protocol heartbeat messages. The fact that it does not show a clear periodicity might be due to the requesting devices not being time-synchronized or by deviating device configurations. Alternatively, heartbeat messages might take place with a very high frequency, so that no periodicities can be observed at the analyzed time scale.

\textbf{Identifying outliers in the application.} 
Fig.~\ref{fig:scores} shows outlier scores of points in time order. The manual inspection of flows with highest outlierness (in the center of Fig.~\ref{fig:scores}) revealed that they are firmware update processes. Since updates took place only during two days in the monitored time span, then high outlier scores are consistent (yet they are not contextual anomalies).

\begin{figure}[b]
	\centering
	\includegraphics[width=0.9\columnwidth]{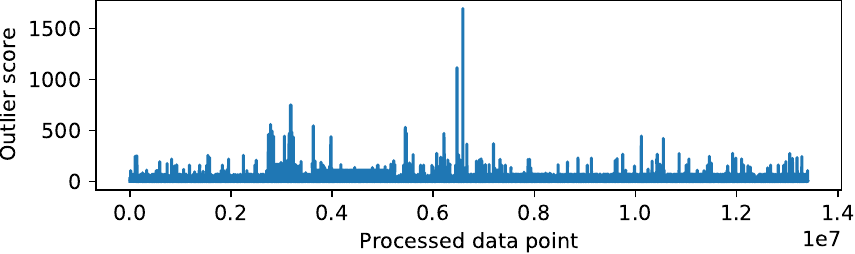}
	\caption{Outlier scores of network data from an e-charging infrastructure.}
	\label{fig:scores}
\end{figure}

\textbf{Learning stability.} 
Finally, we investigated whether our results meet the expected algorithm behavior with respect to the sampling of new observers. Fig.~\ref{fig:sampling} shows how many data points have been sampled as new observers during the first two weeks. With $T=1$ week and $k=400$ observers, $400/7\approx 57$ observers should be sampled each day according to equation~\ref{sample_count}. This theoretical conjecture shows good agreement with the empirical results.
Fig.~\ref{fig:sampling} also shows that the model is not instantly filled with observers in the first hours, but it is instead built up during the first days. Since data seen within the first couple of hours might not be representative for the remaining data, this transient behavior boosts the swift discovery of a representative model while achieving a fast model buildup  with a high sampling rate at the beginning.

\begin{figure}[b]
	\centering
	\includegraphics[width=0.9\columnwidth]{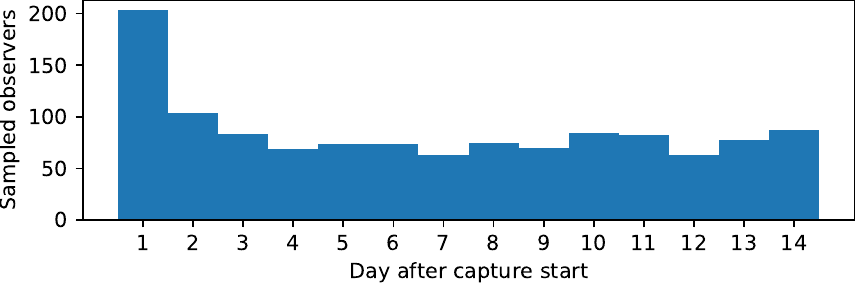}
	\caption{Sampling of arriving data points as new observers when processing network data captured in an e-charging infrastructure.}
	\label{fig:sampling}
\end{figure}

\subsection{Discovery of Temporal Patterns: Darkspace Data}

\textbf{Application context and parameterization.} 
We additionally tested our method on the publicly available CAIDA ``Patch Tuesday'' darkspace dataset \cite{CAIDA2012}. During preprocessing, we aggregated features  by source IP address using the AGM feature vector \cite{Iglesias2017}, specifically proposed for analyzing darkspace data. We applied our algorithm with $T_0=1$ week and $T=10$ weeks and 100 observers and 100 frequency bins, resulting in a minimum period length of about 100 minutes.

\textbf{Capturing and identifying periodical patterns/clusters.} 
Fig.~\ref{fig:darkspace_ft} shows the magnitude of the Fourier coefficients of the three strongest observers. Peaks in Fig.~\ref{fig:darkspace_ft} occur at the 7th and 14th frequency bins, which are diurnal and semi-diurnal periodicities. 
This coincides with previous studies of the darkspace \cite{Iglesias2017} 
that, among others, identified Conficker.C worm attacks or BitTorrent misconfigurations for diurnal patterns, and horizontal scan, vertical scan and probing activities on the UDP protocol for semi-diurnal patterns.

\begin{figure}[t]
	\centering
	\includegraphics[width=0.8\columnwidth]{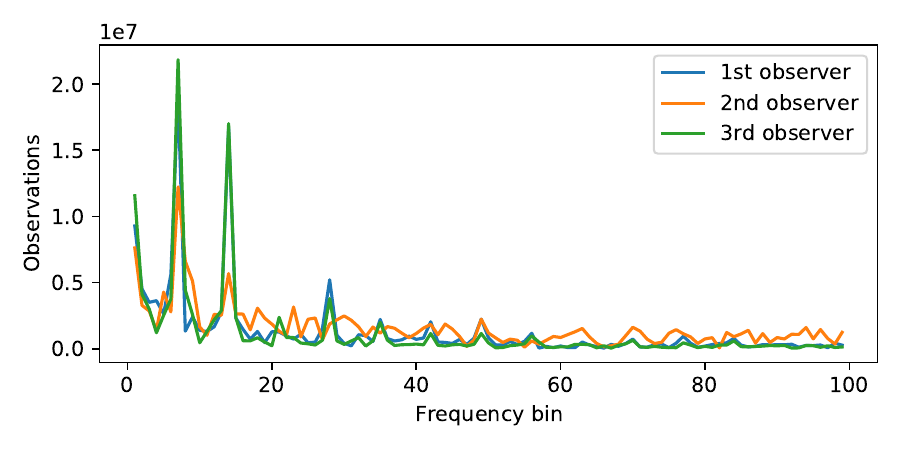}
	\caption{FT of the top three observers after processing darkspace data.}
	\label{fig:darkspace_ft}
\end{figure}

\section{Related Work} \label{sec:related_work}

The problem addressed in this paper is covered in diverse fields, yet differing in some core aspects. 
In this section, we provide an overview of related concepts and precedents to place SDOoop in the landscape of existing work.

\textbf{Time Series Analysis and Contextual Anomalies}. 
A time series is a temporal sequence of observations of specific measurement variables. Time series have been studied in multiple domains, e.g., 
finance and econometrics~\cite{Hamilton2020}
weather forecasting~\cite{Mehalakshmi2016}, 
electric load forecasting~\cite{Sapankevych2009}. 
Traditionally, time series have been analyzed with mathematical tools  
\cite{Hamilton2020}, 
however, in multivariate time series complexity increases dramatically and experts usually resort to nonlinear machine learning, e.g.,~\cite{Mehalakshmi2016, Liao2005}
Multivariate time series and streaming data are frequently considered synonyms, showing small differences open to discussion \cite{Read2020}.

Contextual anomalies have been tackled mainly in the time series analysis domain \cite{Shaukat2021}, but here experts also emphasize the low attention given to contextual anomalies in spite of its relevance for cybersecurity, healthcare sensory and fraud detection \cite{Golmohammadi2015}. 
Latest research tends to expand the focus and, besides point anomalies, face \emph{anomalous sequences} within the whole time series context, giving rise to 
fast model-based detectors \cite{Boniol2021}. However, when considered, contextual anomalies are confronted mainly from a univariate time series perspective. An exception is the recent work by Pasini et al. \cite{Pasini2022}, which copes with low-dimensional multivariate time series and proposes a global \emph{contextual variance} score by weighting \emph{feature-wise contextual variances} based on Mahalanobis distances. 
Note that such approach assumes feature independence though.

\textbf{OD in Streaming Data}
The trend in OD of recent years is to build models to process streaming data with constant memory complexity. In addition to distance-based methods like SDOstream \cite{Hartl2019}, OD in streams is grounded on tree-based methods \cite{Guha2016}, half-space chains \cite{Tan2011},
histograms \cite{Pevny2016}, randomized hashing \cite{Sathe2016}, 
or simply based on nearest neighbors in a sliding window \cite{Angiulli2007}. 
A thorough comparison of these methods can be consulted in \cite{Iglesias2023b}.

When compared to SDOoop, beyond the core approach for calculating point outlierness, the main difference is that
earlier algorithms establish a temporally evolving model (or a set of reference data points) that is deemed stationary at time scales smaller than a pre-fixed time parameter.
Hence, contextual outliers, i.e. data points that occur at an atypical time (out-of-phase), are wrongly classified as normal inliers. 
Another important property of OD methods is interpretability of returned outlier scores.
In fact, many modern techniques like forest-based methods require space transformations that inevitably sacrifice interpretability. 
In SDOoop, outlier scores can be directly interpreted as distance-to-observers, i.e., distance-to-normality. The model is small enough for manual inspection, allowing the analyst to draw conclusions about the data mass based on main model patterns. On the other hand, obtained models are also suitable for stand-alone systems or frameworks where knowledge must be integrated with decision-making modules or other types of knowledge. 

Table~\ref{fig:algo_characteristics} shows a summarized comparison of recent algorithms for evolving stream OD with regard to key properties. SW-$k$NN and SW-LOF denote sliding-window implementations of the popular $k$NN~\cite{Ramaswamy2000} and LOF~\cite{Breunig2000} algorithms.

\newcolumntype{R}[2]{%
	>{\adjustbox{angle=#1,lap=\width-(#2)}\bgroup}%
	l%
	<{\egroup}%
}
\newcommand*\rot{\multicolumn{1}{R{50}{1em}}}
\begin{table}[t]
	\caption{Characteristics of OD algorithms for evolving data streams.}
	\label{fig:algo_characteristics}
	\centering
	\footnotesize
	\begin{tabular}{l c c c c c c c c} 
		& \rot{SW-$k$NN} & \rot{SW-LOF} & \rot{LODA~\cite{Pevny2016}} & \rot{xStream~\cite{Manzoor2018}} & \rot{RS-Hash~\cite{Sathe2016}} & \rot{RRCT~\cite{Guha2016}} & \rot{SDOstream~\cite{Hartl2019}} & \rot{SDOoop}  \\\toprule
		Fixed time complex. & \raisebox{-0.9ex}{$\widetilde{}$} &$\times$&\checkmark&\checkmark&\checkmark&\raisebox{-0.9ex}{$\widetilde{}$} & $\checkmark$ & $\checkmark$ \\
		Fixed space complex. & $\times$ &$\times$&$\times$&\checkmark&\checkmark&$\times$ & \checkmark & $\checkmark$ \\
		Interpretability & \checkmark &\raisebox{-0.9ex}{$\widetilde{}$}&$\times$&$\times$&$\times$&$\times$ & \checkmark & $\checkmark$ \\
		Detect temp. patterns & $\times$ &$\times$&$\times$&$\times$&$\times$&$\times$ & $\times$ & $\checkmark$ \\
		Detect context. anom. & $\times$ &$\times$&$\times$&$\times$&$\times$&$\times$ & $\times$ & $\checkmark$ \\
		\bottomrule
	\end{tabular}
\end{table}

\textbf{Periodic Pattern Mining}. 
The detection of periodicities in 
sequences has also been investigated in the context of periodic pattern mining \cite{Fournier2017}.
Periodic pattern mining can be applied to spatiotemporal data \cite{Zhang2015} to detect periodicities in the movement of objects. 
In contrast, SDOoop is able to detect periodicities of arbitrary clusters  even if the corresponding data points are mixed up with data points from other clusters with different temporal patterns or no patterns at all.
To the best of our knowledge, this problem has not been explored before.

\section{Conclusion} \label{sec:conclusion}
Big data frequently arrives in data streams and requires online processing and analysis. We proposed SDOoop, a method for knowledge discovery in data streams that is able to capture coexisting periodicities regardless of data geometries.
Our method performs a single pass through the data and builds a fixed-size model consisting of representative point locations along with their temporal behavior in Fourier space. 
We showed equal or superior performances compared to state-of-the-art algorithms when testing OD in established evaluation datasets. Moreover, we showed that our method can be an important tool for understanding and visualizing the spatiotemporal behavior of steadily arriving real-world data,
particularly in network security and critical infrastructures communications.

\bibliographystyle{IEEEtran} 
\bibliography{main}      
\input{appendix}

\end{document}

%% file: appendix.tex
\appendices

\section{Proof of Lemma \ref{ft_lemma}} \label{sec:proof1}
Let $i_o \in \mathbb N$ denote the index of a data point, for which $\boldsymbol{\omega}$ is contained in $\mathcal N$ and $i_c$ is the index of the currently processed data point, i.e., $i_o < i_c$. Then, the contribution of $i_o$ to $P_{\boldsymbol{\omega},n}$ according to line 4 of Algorithm \ref{alg:sdostream} has been multiplied by $\Pi _{i=i_o+1}^{i_c}\big(\exp(-T^{-1}+jn2 \pi /{T_0})\big)^{t_i-t_{i-1}} = \big(\exp(-T^{-1}+jn2 \pi /{T_0})\big)^{t_{i_c}-t_{i_o}}$. Summing over all points that have arrived in $\boldsymbol{\omega}$'s neighborhood, we can write 
\begin{align*}
E\{P_{\boldsymbol{\omega},n}\} = \int _{-\infty}^t g(\tau)  \Big(\exp(-T^{-1}+jn2 \pi /{T_0})\Big)^{t-\tau} d\tau.
\end{align*}
Splitting the integral into intervals of length $T_0$, we obtain
\small
\begin{align*}
& E\{P_{\boldsymbol{\omega},n}\} =  
 \sum_{l=0}^{\infty} \int _{t\text{-}T_0}^t g(\tau\text{-}lT_0)  \big(\exp(\text{-}T^{\text{-}1}+jn2 \pi /{T_0})\big)^{t\text{-}\tau\text{-}lT_0} d\tau  \\
& = \Big( \sum_{l=0}^{\infty} \exp(\text{-}T^{\text{-}1}lT_0) \Big) \int _{t\text{-}T_0}^t  g(\tau)  \big(\exp(\text{-}T^{\text{-}1}+jn2 \pi /{T_0})\big)^{t\text{-}\tau} d\tau 
\end{align*}
\normalsize
due to $T_0$-periodicity of $g(t)$ and $\exp(jn2\pi)=1$. Abbreviating the constant factor and substituting $\tau^{\prime}=\tau-t$, we obtain
\begin{align*} 
& E\{P_{\boldsymbol{\omega},n}\} = 
 c \int _{-T_0}^0 g(\tau^{\prime}-t)  (\exp(-T^{-1}+jn2 \pi /{T_0}))^{-\tau^{\prime}} d\tau^{\prime} \\ 
& \stackrel{T \gg T_0}{\approx} 
c \int _{-T_0}^0 g(\tau^{\prime}-t)  \exp(-jn2 \pi \tau^{\prime} /{T_0})) d\tau^{\prime} .  \quad \IEEEQEDhere 
\end{align*}

\section{Proof of Theorem \ref{sample_count}} \label{sec:proof2}
Taking line 6 in Algorithm \ref{alg:sdostream} as starting point, 
the probability of selecting a newly seen point as observer is
$\min\left(1,\frac{k^2}{Tx} \frac{\sum_{\boldsymbol{\omega} \in \mathcal N} P_{\boldsymbol{\omega},0}}{\sum_{\boldsymbol{\omega} \in \Omega} P_{\boldsymbol{\omega},0}} \frac{t_i-t_{i_{\text{LAO}}}}{i-i_{\text{LAO}}}\right)$. 
Since we target specifically data streams with high rates of arriving data points, we can safely assume this probability to be small.
Hence, $\Pr\left\{1 <\frac{k^2}{Tx} \frac{\sum_{\boldsymbol{\omega} \in \mathcal N} P_{\boldsymbol{\omega},0}}{\sum_{\boldsymbol{\omega} \in \Omega} P_{\boldsymbol{\omega},0}} \frac{t_i-t_{i_{\text{LAO}}}}{i-i_{\text{LAO}}}\right\}$ is negligible and we can write for the average probability of sampling a new point as observer
$P_s \approx E \left\{ \frac{k^2}{Tx} \frac{\sum_{\boldsymbol{\omega} \in \mathcal N} P_{\boldsymbol{\omega},0}}{\sum_{\boldsymbol{\omega} \in \Omega} P_{\boldsymbol{\omega},0}} \frac{t_i-t_{i_{\text{LAO}}}}{i-i_{\text{LAO}}}\right\}$. Under the same assumption, we observe that the term $\frac{t_i-t_{i_{\text{LAO}}}}{i-i_{\text{LAO}}}$ depends on the current time, but, since points belonging to different neighborhoods arrive in an interleaved manner, does not depend on a point's neighborhood. Since $\frac{\sum_{\boldsymbol{\omega} \in \mathcal N } P_{\boldsymbol{\omega},0}}{\sum_{\boldsymbol{\omega} \in \Omega} P_{\boldsymbol{\omega},0}}$ does not depend on time, we can split the term to $P_s \approx E \left\{ \frac{k^2}{Tx} \frac{\sum_{\boldsymbol{\omega} \in \mathcal N } P_{\boldsymbol{\omega},0}}{\sum_{\boldsymbol{\omega} \in \Omega} P_{\boldsymbol{\omega},0}} \right\} E\left\{ \frac{t_i-t_{i_{\text{LAO}}}}{i-i_{\text{LAO}}}\right\}$ due to stochastic independence of both terms.
$\sum_{\boldsymbol{\omega} \in \mathcal N } P_{\boldsymbol{\omega},0} /x$ expresses the average observation count in the current neighborhood. The algorithm implements several mechanisms to make the observer density agree with the time-averaged point density, rendering the time-averaged local average observation count $E\left\{\sum_{\boldsymbol{\omega} \in \mathcal N } P_{\boldsymbol{\omega},0} /x\right\}$ equal to the total average observation count of all observers $\sum_{\boldsymbol{\omega} \in \Omega } P_{\boldsymbol{\omega},0} /k$, hence $P_s \approx \frac{k}{T}E\left\{\frac{t_i-t_{i_{\text{LAO}}}}{i-i_{\text{LAO}}}\right\}=\frac{k}{T}\overline{\text{IAT}}$, where the average inter-arrival time of two data points is termed $\overline{\text{IAT}}$.
During a time period of $T$, $T/\overline{\text{IAT}}$ data points arrive, yielding an average number of sampled points of $P_sT/\overline{\text{IAT}}=k$. \, \IEEEQEDhere